\documentclass[runningheads]{llncs}
\usepackage{graphicx}
\usepackage{amsmath}
\usepackage{amssymb}
\usepackage{enumerate}

\usepackage{amsthm}
\DeclareMathOperator{\R}{\mathbb{R}}

\DeclareMathOperator{\Z}{\mathbb{Z}}
\DeclareMathOperator{\N}{\mathbb{N}}
\DeclareMathOperator{\Proj}{\mathbb{P}}
\DeclareMathOperator{\Lplus}{\Lambda_+}
\DeclareMathOperator{\dash}{\;\cdot\;}
\newenvironment{proofsketch}[1]{\par\noindent\textit{Proof sketch.}\space#1}{\hfill $\square$}
\newenvironment{prooflem}[1]{\par\noindent\textit{Proof.}\space#1}{\hfill $\blacksquare$}
\begin{document}
\title{The Gaussian kernel on the circle and spaces that admit isometric embeddings of the circle}
\titlerunning{The Gaussian kernel on the circle}
%
\author{Nathaël Da Costa \and
Cyrus Mostajeran \and
Juan-Pablo Ortega}
\authorrunning{N. Da Costa et al.}
%
\institute{Division of Mathematical Sciences, Nanyang Technological University, Singapore 637371}
\maketitle              
\begin{abstract}
On Euclidean spaces, the Gaussian kernel is one of the most widely used kernels in applications. It has also been used on non-Euclidean spaces, where it is known that there may be (and often are) scale parameters for which it is not positive definite. Hope remains that this kernel is positive definite for many choices of parameter. However, we show that the Gaussian kernel is not positive definite on the circle for any choice of parameter. This implies that on metric spaces in which the circle can be isometrically embedded, such as spheres, projective spaces and Grassmannians, the Gaussian kernel is not positive definite for any parameter.
\end{abstract}
\section{Introduction}
In many applications, it is useful to capture the geometry of the data and view it as lying in a non-Euclidean space, such as a metric space or a  Riemannian manifold. Examples of such applications include computer vision \cite{romeny_geometry-driven_2013}, robot learning \cite{calinon_gaussians_2020} and brain-computer interfaces \cite{barachant_riemannian_2010}. We are interested in the problem of applying kernel methods on such non-Euclidean spaces. \par
Kernel methods are prominent in machine learning, with some examples of algorithms including support vector machines \cite{cristianini_support_2008}, kernel principal component analysis \cite{scholkopf_nonlinear_1998}, solvers for controlled and stochastic differential equations \cite{salvi2021signature}, and reservoir computing \cite{rc16}, \cite{rc25}. These algorithms rely on the existence of a reproducing kernel Hilbert space into which the kernel maps the data. This in turn requires the chosen kernel to be positive definite (PD). \par
One of the most common types of kernel used in applications is the Gaussian kernel. Defined on a Euclidean space, this kernel is PD for any choice of parameter. \cite{schoenberg_metric_1938} shows moreover that the Gaussian kernel defined on a metric space is PD for all parameters if and only if the metric space can be isometrically embedded into an inner product space. This implies that Euclidean spaces are the only complete Riemannian manifolds for which the Gaussian kernel is PD for all parameters \cite{feragen_geodesic_2014}, \cite{jayasumana_kernel_2015}.
However, the problem of determining for which parameters the Gaussian kernel is PD on a given metric space is not solved. \cite{sra_positive_2013} shows that the Gaussian kernel may be PD for a wide range of parameters even when it is not PD for every parameter. However, this paper rules out such a possibility for a large class of spaces of interest. \par
We start by defining positive definite kernels. Then we give a brief review of the literature on the positive definiteness of the Gaussian kernel, and introduce some new notation to study this problem. Finally, we show that the Gaussian kernel is not PD for any choice of parameter on the circle, and consequently for any metric space admitting an isometrically embedded circle. \par
We should note that since producing the results of this paper, we have discovered that certain general characterisations of positive definite functions on the circle exist in the literature, which encompass our result on the circle \cite{wood_when_1995}, \cite{gneiting_strictly_2013}. Our proof, however, is specific to the Gaussian kernel, relies only on elementary analysis, and provides further insight into the extent to which the Gaussian kernel fails to be PD, which may have practical relevance for applications of kernel methods to non-Euclidean data processing.
\section{Kernels}
\begin{definition}
A kernel on a set $X$ is a symmetric map
$$k:X\times X \rightarrow \R.$$
$k$ is said to be positive definite (PD) if for all $N\in\N$, $x_1,\dots,x_N\in X$ and all $c_1,\dots,c_N\in\R$,
$$\sum_{i=1}^{N}\sum_{j=1}^{N}c_ic_jk(x_i,x_j)\geq 0$$
i.e. the matrix $\big(k(x_i,x_j)\big)_{i,j}$, which we call the Gram matrix of $x_1,\dots,x_N$, is positive semi-definite.
\end{definition}
\begin{proposition}\label{cvexcone}
Suppose the $(k_n)_{n\geq 1}$ are PD kernels on $X$.
\begin{enumerate}[(i)]
    \item $a_1k_1+a_2k_2$ is a PD kernel on $X$ for all $a_1,a_2\geq0$.
    \item The Hadamard (pointwise) product $k_1\cdot k_2$ is a PD kernel on $X$.
    \item If $k_n\rightarrow k$ pointwise as $n\rightarrow\infty$, then $k$ is a PD kernel on $X$.
    \item If $Y\subset X$ then $k_1|_Y$ is a PD kernel on $Y$.
\end{enumerate}
\end{proposition}
\begin{proof}
The $N\times N$ symmetric positive semidefinite matrices $\text{Sym}^{0+}(N)$ form a closed convex cone in the space of symmetric matrices $\text{Sym}(N)$, which implies (i) and (iii). $\text{Sym}^{0+}(N)$ is also closed under pointwise multiplication, as shown in \cite[Chapter 3 Theorem 1.12.]{berg_harmonic_1984}, which implies (ii). Finally, proving (iv) is trivial.
\end{proof}

Proposition \ref{cvexcone} (i), (ii), and (iii) say that PD kernels on $X$ form a convex cone, closed under pointwise convergence and pointwise multiplication.
\section{The Gaussian kernel}\label{s2}
In this section, $X$ is a metric space equipped with the metric $d$. A common type of kernel on such a space is the Gaussian kernel
\begin{equation}\label{GK}
k(\dash,\dash) := \exp(-\lambda d(\dash,\dash)^2)
\end{equation}
where $\lambda>0$. Write
$$\Lplus(X) := \{\lambda>0: \text{the Gaussian kernel with parameter $\lambda$ is PD} \}.$$
We would like to characterise $\Lplus(X)$ in terms of $X$. In what follows, Propositions \ref{cvexconeG} and \ref{subsetsG} are the analogous to Proposition \ref{cvexcone} for Gaussian kernels.
\begin{proposition}\label{cvexconeG}

(i) $\Lplus(X)$ is closed under addition. \\
(ii) $\Lplus(X)$ is topologically closed in $(0,\infty)$.

\end{proposition}
\begin{proof}
(i) and (ii) follow from Proposition \ref{cvexcone} (ii) and (iii) respectively.
\end{proof}
\begin{corollary}\label{corcharG}
(i) If there is $\epsilon>0$ s.t. $(0,\epsilon)\subset\Lplus(X)$ then $\Lplus(X)=(0,\infty).$ \\
(ii) If there is $\epsilon>0$ s.t. $\Lplus(X)\subset(0,\epsilon)$ then $\Lplus(X)=\varnothing$.
\end{corollary}
\begin{proof}
These both follow from Proposition \ref{cvexconeG} (i).
\end{proof}
\begin{definition}\label{embedding}
Let $Y$ be another metric space with metric $d'$. We say $Y$ isometrically embeds into $X$, written $Y\hookrightarrow X$ if there is a function $\iota: Y\rightarrow X$ such that
$$d(\iota(\dash),\iota(\dash))=d'(\dash,\dash).$$
\end{definition}
Note that, while the notion of `isometry' in the context of Riemannian manifolds and in the context of metric spaces correspond (Myers–Steenrod theorem), the notion of `isometric embedding' is stronger in the context of metric spaces than in the context of Riemannian manifolds. For example, the unit 2-sphere $S^2$ can be isometrically embedded in $\R^3$ in the sense of Riemannian manifolds, but not in our sense.
\begin{proposition}\label{subsetsG}
Let $Y$ be another metric space with metric $d'$. If $Y\hookrightarrow X$, then $\Lplus(X)\subset\Lplus(Y)$.
\end{proposition}
\begin{proof}
Follows from Proposition \ref{cvexcone} (iv).
\end{proof}

As of now, we have only made rather elementary observations about $\Lplus(X)$, but now we state the first major result, without proof.
\begin{theorem}[due to I.J. Schoenberg \cite{schoenberg_metric_1938}]\label{mainthm}
The following are equivalent:
\begin{enumerate}
    \item $\Lplus(X) = (0,\infty)$.
    \item $X\hookrightarrow \mathcal{V}$ for some inner product space $\mathcal{V}$.
\end{enumerate}
\end{theorem}
Note that, if it exists, the isometric embedding $X\hookrightarrow\mathcal{V}$ is not in general related to the reproducing kernel Hilbert space (RKHS) map for the Gaussian kernel. Given a positive definite kernel $k$ on $X$, the RKHS map is a set-theoretic map $K: X\rightarrow \mathcal{H}$ where $\mathcal{H}$ is a Hilbert space such that $k(\dash,\dash)=\langle K(\dash),K(\dash)\rangle$. These are different objects. \par
Theorem \ref{mainthm} is already very powerful, and guarantees that $\Lplus(X)=(0,\infty)$ for many spaces.
\begin{corollary}\label{consequences}
$\Lplus(X)=(0,\infty)$ for the following spaces $X$:
\begin{enumerate}
    \item $\R^n$ with the Euclidean metric, for $n\geq1$.
    \item $L^2_{\R} (\Omega,\mu)$ for any measure space $(\Omega,\mu)$.
    \item $\text{Sym}^{++}(n)$ the space of symmetric $n\times n$ positive definite matrices, with the Frobenius metric, for $n\geq 1$.
    \item $\text{Sym}^{++}(n)$ with the log-Euclidean metric $d(A,B)=\|\log(A)-\log(B)\|_F$, for $n\geq 1$, where $\|\cdot\|_F$ denotes the Frobenius norm.
    \item $Gr_{\R} (k,n)$ the real Grassmanian with the projection metric $d([A],[B])=\|AA^T-BB^T\|_F$, where $A,B$ are the $n\times k$ matrices representing the subspaces $[A],[B]$ respectively, for $1\leq k<n$.
\end{enumerate}
\end{corollary}
\begin{proof}
Follows directly from Theorem \ref{mainthm}.
\end{proof}
Moreover, \cite{feragen_geodesic_2014} and \cite{jayasumana_kernel_2015} deduce from Theorem \ref{mainthm} the following result, which we state without proof.
\begin{theorem}\label{mfdthm}
If $X$ is a complete Riemannian manifold, $\Lplus(X) = (0,\infty)$ if and only if $X$ is isometric to a Euclidean space.
\end{theorem}
While Theorem \ref{mainthm} is powerful, the full characterisation of $\Lplus(X)$ is far from solved. $\Lplus(X)$ can be non-empty and different from $(0,\infty)$; it is easy to construct finite metric spaces with more complicated $\Lplus(X)$. This can also be the case for more complex metric spaces: \cite[Theorem 3.10]{sra_positive_2013} shows that on the space of symmetric positive definite matrices $X=\text{Sym}^{++}(n)$ equipped with the metric of Stein divergence, we have
$$\Lplus(X) = \Big\{\frac{1}{2},\frac{2}{2},\dots,\frac{n-2}{2}\Big\}\cup\Big[\frac{n-1}{2},\infty\Big).$$
While this result gives hope that the Gaussian kernel may be PD for many parameters on many interesting spaces, we show that this is often not the case.
\section{The Gaussian kernel on the circle}
\begin{theorem}\label{newthm}
$\Lplus(S^1) = \varnothing$ where $S^1$ is the unit circle with its classical intrinsic metric.
\end{theorem}
\begin{proof}
Let $N\in\N$. Define $x_k = 2\pi k/N$ for $0\leq k\leq N-1$. So $$d(x_k,x_l) = \frac{2\pi}{N}\min\{|k-l+mN|: m\in\Z\}$$
for all $0\leq k,l\leq N-1$.
\begin{figure}
\centerline{\includegraphics[scale=.1]{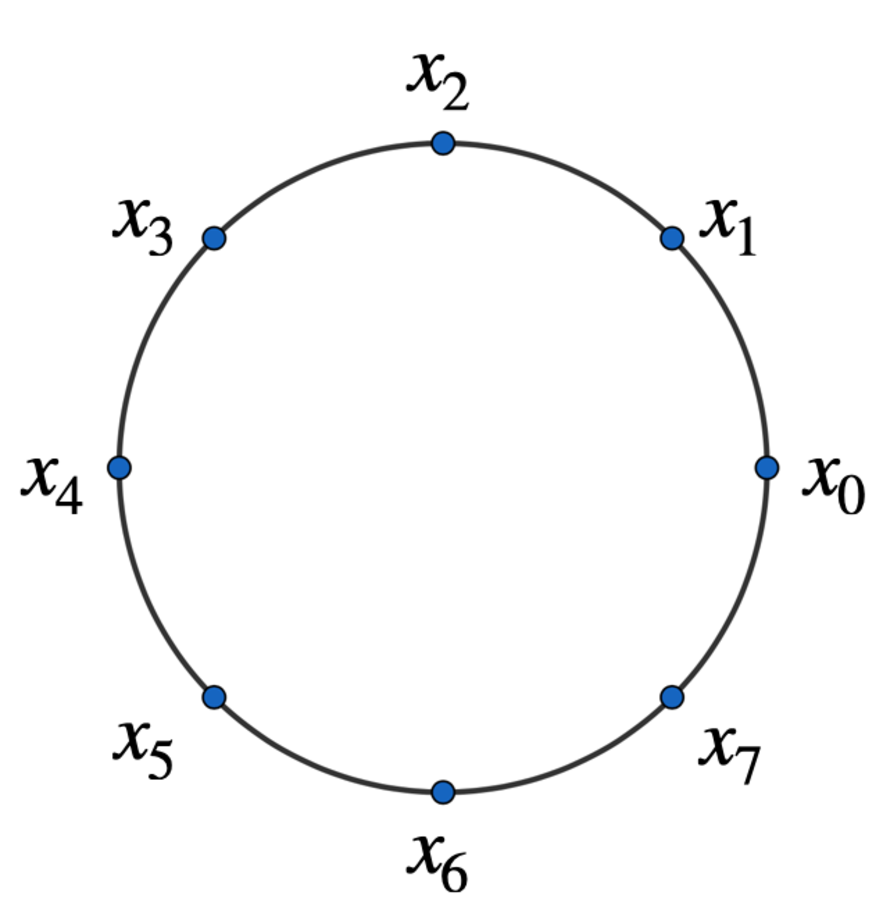}}
\caption{Points $x_k$ on $S^1$ for $N=8$.}
\label{fig}
\end{figure}
\par
So the Gram matrix of $x_0,\dots,x_{N-1}$ is
\begin{equation*}
K = \begingroup
\setlength\arraycolsep{6pt}
\begin{bmatrix} 
    1 & \exp(-\lambda(\frac{2\pi}{N})^2) & \exp(-\lambda(\frac{4\pi}{N})^2) & \dots & \exp(-\lambda(\frac{2\pi}{N})^2) \\
    \\
    \exp(-\lambda(\frac{2\pi}{N})^2) & 1 & \dots  & & \exp(-\lambda(\frac{4\pi}{N})^2) \\
    \\
    \\
    \vdots & \vdots & \ddots & & \vdots \\
    \\
    \\
    \exp(-\lambda(\frac{2\pi}{N})^2) & \exp(-\lambda(\frac{4\pi}{N})^2) & \exp(-\lambda(\frac{6\pi}{N})^2) & \dots & 1
    \end{bmatrix}.
    \endgroup
    \end{equation*}
To show the Gaussian kernel with parameter $\lambda$ is not PD all we need to show is that we can choose $N$ such that $K$ has a negative eigenvalue. \par
$K$ is a circulant matrix, so its eigenvalues are given by the discrete Fourier transform of the first row. Explicitly, these eigenvalues are
$$w_j = \sum_{k=0}^{\lfloor N/2\rfloor}\exp\left(-\mu\frac{k^2}{N^2}\right)e^{i\frac{2\pi}{N}kj}+\sum_{k=\lfloor N/2\rfloor+1}^{N-1}\exp\left(-\mu\frac{(N-k)^2}{N^2}\right)e^{i\frac{2\pi}{N}kj}$$
for $0\leq j\leq N-1$, where $\mu=\lambda(2\pi)^2$. Taking $N\equiv 0 \;\text{mod}\;2$, this gives
\begin{equation*}
\begin{aligned}
w_j &= 1+\sum_{k=1}^{N/2-1}\exp\left(-\mu\frac{k^2}{N^2}\right)(e^{i\frac{2\pi}{N}kj}+e^{i\frac{2\pi}{N}(N-k)j})+\exp(-\mu/4)e^{i\pi j} \\
&= -1+2\sum_{k=0}^{N/2-1}\exp\left(-\mu\frac{k^2}{N^2}\right)\cos\left(\frac{2\pi}{N}kj\right)+(-1)^j\exp(-\mu/4).
\end{aligned}
\end{equation*}
Restricting further to $N\equiv 0 \;\text{mod}\;4$ and $j=N/2$, the sum conveniently becomes alternating:
\begin{equation}\label{evalue}
w_{N/2} = -1+2\underbrace{\sum_{k=0}^{N/2-1}(-1)^k\exp\left(-\mu\frac{k^2}{N^2}\right)}_{(*)}+\exp(-\mu/4).
\end{equation}
We will show that $w_{N/2}$ is negative for $N$ large enough. For this, we need to estimate the second term of (\ref{evalue}). The difficulty lies in the fact that the variable $N$ appears in both the terms and the indices of the sum. To remedy this we define
$$S_r(N) := \sum_{k=0}^{\infty}(-1)^k\exp\left(-\mu\frac{k^2}{N^2}-r\frac{k}{N}\right)$$
for $r\geq0$. These series are instances of partial theta functions, and below we leverage two facts about them from the literature. But first, let us express $w_{N/2}$ in terms of these. We have
\begin{equation}\label{series}
\begin{aligned}
S_0(N) &= \sum_{k=0}^{\infty}(-1)^k\exp\left(-\mu\frac{k^2}{N^2}\right) \\
=\underbrace{\sum_{k=0}^{N/2-1}(-1)^k\exp\left(-\mu\frac{k^2}{N^2}\right)}_{(*)}&+\exp(-\mu/4)-\exp(-\mu/4)\exp\left(-\mu\frac{1}{N^2}-\mu\frac{1}{N}\right) \\
&+\underbrace{\sum_{k=N/2+2}^{\infty}(-1)^k\exp\left(-\mu\frac{k^2}{N^2}\right)}_{(**)}.
\end{aligned}
\end{equation}
We remove the dependency on $N$ from the indices of $(**)$:
\begin{equation}\label{lastterm}
\begin{aligned}[b]
\sum_{k=N/2+2}^{\infty}(-1)^k&\exp\left(-\mu\frac{k^2}{N^2}\right)=\sum_{k=0}^{\infty}(-1)^{k+N/2}\exp\left(-\mu\frac{(k+N/2+2)^2}{N^2}\right)\\
&= \exp\left(-\mu\frac{(N/2+2)^2}{N^2}\right)\sum_{k=0}^{\infty}(-1)^k\exp\left(-\mu\frac{k^2}{N^2}-\mu\left(1+\frac{4}{N}\right)\frac{k}{N}\right) \\
&= \exp(-\mu/4)\exp\left(-\mu\frac{4}{N^2}-\mu\frac{2}{N}\right)S_{\mu(1+4/N)}(N).
\end{aligned}
\end{equation}
Substituting (\ref{lastterm}) into (\ref{series}), and in turn substituting (\ref{series}) into (\ref{evalue}) we get
\begin{equation*}
\begin{aligned}
w_{N/2}= -1+2S_0(N)+\exp(-\mu/4)\bigg(&-1+2\exp\left(-\mu\frac{1}{N^2}-\mu\frac{1}{N} \right) \\
&-2\exp\left(-\mu\frac{4}{N^2}-\mu\frac{2}{N}\right)S_{\mu(1+4/N)}(N)\bigg).
\end{aligned}
\end{equation*}
Now we use the following lemma.
\begin{lemma}\label{f1}
$$S_r(N)\geq S_0(N)$$
for all $r\geq 0$ and for all $N$.
\end{lemma}
\begin{prooflem}
This follows from \cite[Proposition 14 Equation 5.8]{carneiro_bandlimited_2013}.
\end{prooflem}
\\
\par
So
\begin{equation}\label{estimate}
\begin{aligned}
w_{N/2}\leq -1+2S_0(N)+\exp(-\mu/4)\bigg(&-1+2\exp\left(-\mu\frac{1}{N^2}-\mu\frac{1}{N} \right)\\
&-2\exp\left(-\mu\frac{4}{N^2}-\mu\frac{2}{N}\right)S_0(N)\bigg).
\end{aligned}
\end{equation}
The limit of the RHS of (\ref{estimate}) as $N\rightarrow\infty$ is 0, so it is not enough just to take the limit. Instead, we will need to take an asymptotic expansion with respect to $1/N$ to the second order. For this, we need a second lemma.
\begin{lemma}\label{f2}
$$S_0(N) = \frac{1}{2}+O(1/N^n)\text{ as } N\rightarrow\infty$$
for all $n\geq 1.$
\end{lemma}
\begin{prooflem}
\cite[Theorem 1.1 (i)]{bringmann_asymptotic_2017} says that for $n\geq1$,
$$S_0(N) = \sum_{a=0}^n \bigg(\frac{1}{2\pi i}\frac{\partial}{\partial z}\bigg)^{2a}\bigg[\frac{1}{1-e^{2\pi i z}}\bigg]_{z=1/2}\frac{(-\mu)^a}{a!}\frac{1}{N^{2a}}+O(1/N^{2n+1})\text{ as } N\rightarrow\infty.$$
Now observe
$$f(z)   := \frac{1}{1-e^{2\pi i(z-1/2)}}-\frac{1}{2}=\frac{i}{2}\tan(\pi z)$$
is odd, so the even terms in the Taylor series of $f$ vanish, and hence
\begin{equation*}
\bigg(\frac{1}{2\pi i}\frac{\partial}{\partial z}\bigg)^{2a}\bigg[\frac{1}{1-e^{2\pi i z}}\bigg]_{z=1/2}=\begin{cases}
    \frac{1}{2} & \text{if $a=0$}\\
    0 & \text{if $a\geq 1$}
    \end{cases}
\end{equation*}
which gives us the fact.
\end{prooflem}
\\
\par
So taking the asymptotic expansion with respect to $1/N$ to second order, (\ref{estimate}) simplifies to
$$w_{N/2}\leq \exp(-\mu/4)\frac{2\mu-\mu^2}{N^2}+O(1/N^3) \text{ as } N\rightarrow\infty.$$
If $\lambda>\frac{1}{2\pi^2}$ then $\mu > 2$ so $2\mu-\mu^2<0$ and hence $w_{N/2}$ is negative for $N$ large enough, with $N\equiv 0 \;\text{mod}\;4$. It is possible to improve these inequalities to obtain the result for all $\lambda$, although this is unnecessary: Corollary \ref{corcharG} (ii) is enough to conclude the proof.
\end{proof}

Thanks to Proposition \ref{subsetsG}, Theorem \ref{newthm} gives us much more than one may suspect at first.
\begin{corollary} If $S^1\hookrightarrow X$ then $\Lplus(X) = \varnothing$. So $\Lplus(X) = \varnothing$ for the following spaces, equipped with their classical intrinsic metric:
\begin{enumerate}
    \item $S^n$ the sphere, for $n\geq 1$.
    \item $\R\Proj^n$ the real projective space, for $n\geq 1$.
    \item $Gr_{\R} (k,n)$ the real Grassmannian, for $1\leq k < n$.
\end{enumerate}
\end{corollary}
\begin{proofsketch}
This follows from Theorem \ref{newthm} and Proposition \ref{subsetsG}. Now we briefly argue for the specific examples. For 1., $S^1\hookrightarrow S^n$ (e.g., a ‘great circle’). For 2., $1/2\cdot S^1\cong \R\Proj^1\hookrightarrow \R\Proj^n$ where `$\cong$' means isometric and `$1/2\cdot S^1$' means $S^1$ rescaled by a factor of $1/2$. This factor does not affect the conclusion. For 3., the metric in question is
$$d([A],[B]) = \bigg(\sum_{i=1}^k\theta_i^2\bigg)^{1/2}$$
where $\theta_i$ is the $i$-th principal angle between $[A]$ and $[B]$ (see \cite{ye_schubert_2016} and \cite{zhu_angles_2013}), $[A],[B]\in Gr_{\R}(k,n)$. Fixing any $[A]\in Gr_{\R}(k,n)$, travelling on $Gr_{\R}(k,n)$ while keeping $\theta_i=0$ for $i>1$ and varying $\theta_1$ only, we get an isometric embedding of $1/2\cdot S^1$ into $Gr_{\R}(k,n)$.
\end{proofsketch}
\\
\par
It is conceivable that $S^1$ can be isometrically embedded (in the metric sense from Definition \ref{embedding}) into any compact Riemannian manifold (up to rescaling). We have yet to think of a counterexample. If this is true, then Theorem \ref{newthm} would solve the problem of characterising $\Lplus(X)$ for all compact Riemannian manifolds. However, while the Lyusternik–Fet theorem tells us that any compact Riemannian manifold has a closed geodesic, it appears to be an open question whether any such manifold admits an isometric embedding of $S^1$.
\begin{figure}
\centerline{\includegraphics[width=\textwidth]{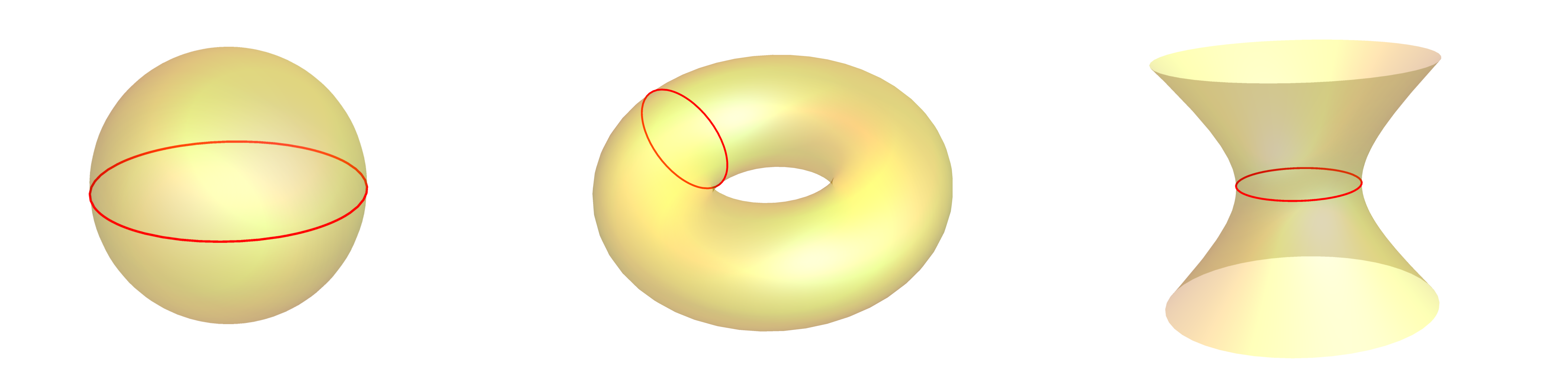}}
\caption{Examples of Riemannian manifolds that admit isometric embeddings of $S^1$. From left to right: a sphere, a torus, a hyperbolic hyperboloid.}
\label{examples}
\end{figure}
\par
Note that non-compact manifolds may also admit isometric embeddings of $S^1$: consider a hyperbolic hyperboloid.
There is precisely one (scaled) isometric embedding of $S^1$ into it. This example is particularly interesting since, as opposed to the examples above with positive curvature, it has everywhere negative curvature. See Figure \ref{examples}.

\section{Discussion}
In machine learning, most kernel methods rely on the existence of an RKHS embedding. This, in turn, requires the chosen kernel to be positive definite. Theorem \ref{newthm} shows that the Gaussian kernel defined in this work cannot provide such RKHS embeddings of the circle, spheres, and Grassmannians. It reinforces the conclusion from Theorem \ref{mfdthm} that one should be careful when using the Gaussian kernel in the sense defined in this work on non-Euclidean Riemannian manifolds. The authors in \cite{borovitskiy_matern_2022} propose a different way to generalise the Gaussian kernel from Euclidean spaces to Riemannian manifolds by viewing it as a solution to the heat equation. This produces positive definite kernels by construction. \par
Nevertheless, perhaps we should not be so fast to altogether reject our version of the Gaussian kernel, which has the advantage of being of particularly simple form. It is worth noting that the proof of Theorem \ref{newthm} relies on taking $N\to\infty$, where $N$ is the number of data points. \cite{feragen_open_nodate} lists three open problems regarding the positive definiteness of the Gaussian kernels on metric spaces. It suggests that we should not only look at whether the Gaussian kernel is PD on the whole space but whether there are conditions on the spread of the data such that the Gram matrix of this data is PD. Our proof of Theorem \ref{newthm} relying on the assumption of infinite data suggests that this may be the case. In general, fixing $N$ data points, the Gram matrix with components $\exp(-\lambda d(\dash,\dash)^2)$ tends to the identity as $\lambda\to\infty$, so will be PD for $\lambda$ large enough. This observation has supported the use of the Gaussian kernel on non-Euclidean spaces, for example, in \cite{jaquier_bayesian_2019} where it is used on spheres.
However, it is important to keep Theorem \ref{newthm} in mind in applications where the data is not fixed, and we need to be able to deal with new and incoming data, which is often the case.
\bibliographystyle{splncs04}
\bibliography{bibliography}

\end{document}